\documentclass[letterpaper, 10 pt, conference]{ieeeconf} 
\IEEEoverridecommandlockouts 
\usepackage{amssymb}
\usepackage{amsmath}
\usepackage{graphicx}
\usepackage{caption}
\usepackage{subcaption}
\graphicspath{{images/}}
\usepackage{float}
\usepackage{hyperref}
\usepackage{algorithm,algcompatible}
\newcommand{\norm}[1]{\left\lVert#1\right\rVert}
\usepackage{algorithmicx}
\usepackage{algpseudocode}
\let\labelindent\relax
\usepackage{enumitem}

\newcommand{\R}{\mathbb{R}}

\newcommand{\s}{\mathbf{s}}

\newcommand{\E}{\mathbb{E}}

\newtheorem{theorem}{Theorem}

\newtheorem{remark}{Remark}

\usepackage{lipsum} 

\title{\LARGE \bf Post-Convergence Sim-to-Real Policy Transfer: A Principled \\ Alternative to Cherry-Picking
}

\author{Dylan Khor$^{1}$ and Bowen Weng$^{1}$
\thanks{$^{1}$Department of Computer Science, Iowa State University, IA, USA;  {\tt\footnotesize dkhor@iastate.edu, bweng@iastate.edu.}}%
}

\begin{document}
\maketitle
\thispagestyle{empty}
\pagestyle{empty}

\begin{abstract}
Learning-based approaches, particularly reinforcement learning (RL), have become widely used for developing control policies for autonomous agents, such as locomotion policies for legged robots. RL training typically maximizes a predefined reward (or minimizes a corresponding cost/loss) by iteratively optimizing policies within a simulator. Starting from a randomly initialized policy, the empirical expected reward follows a trajectory with an overall increasing trend. While some policies become temporarily stuck in local optima, a well-defined training process generally converges to a reward level with noisy oscillations. However, selecting a policy for real-world deployment is rarely an analytical decision (i.e., simply choosing the one with the highest reward) and is instead often performed through trial and error. To improve sim-to-real transfer, most research focuses on the pre-convergence stage, employing techniques such as domain randomization, multi-fidelity training, adversarial training, and architectural innovations. However, these methods do not eliminate the inevitable convergence trajectory and noisy oscillations of rewards, leading to heuristic policy selection or cherry-picking. This paper addresses the post-convergence sim-to-real transfer problem by introducing a worst-case performance transference optimization approach, formulated as a convex quadratic-constrained linear programming problem. Extensive experiments demonstrate its effectiveness in transferring RL-based locomotion policies from simulation to real-world laboratory tests.
\end{abstract}

\section{INTRODUCTION}
Fig.~\ref{fig:train_logs} (b) illustrates the average reward trajectory from training a locomotion policy for the Unitree G1 humanoid robot in Isaac Gym using reinforcement learning (RL)~\cite{unitreegithub} with the random seed being 50. Initially, the randomly initialized policy yields a low training reward. As training progresses, the reward increases steadily, peaking at approximately the 3000th iteration before experiencing a slight decline. It then climbs back and stabilizes within a bounded interval for the remainder of the training process (up to the 20000th iteration).

Although this result is based on a specific random seed for the given experiment, the overall trend is widely observed across various RL applications. Similar training patterns emerge in different setups, reinforcing the generality of this behavior. Note Fig.~\ref{fig:train_logs} (a) with random seed being 1 shows a slightly different behavior towards the second half of the training due to the loss of plasticity as identified by Dohare et al.~\cite{dohare2024loss} (though details are out of the scope of this paper). In practice, the training is typically considered successful once the reward stabilizes at a reasonably large average reward (either permanently as shown in Fig.~\ref{fig:train_logs} (b) or temporarily as shown in Fig.~\ref{fig:train_logs} (a)). However, determining which policy to select for real-world deployment remains an open question.

In practice, this is where most heuristic ``cherry-picking" occurs. To the best of the authors' knowledge, while no specific publications explicitly address this step, it is commonly carried out through an iterative trial-and-error process between the simulator and the real-world. Practitioners typically evaluate policies with sufficiently high empirical rewards, selecting one that performs well in real-world testing. However, this approach lacks a systematic framework, making the sim-to-real policy selection largely ad hoc and unpredictable.

\begin{remark}
    One factor contributing to cherry-picking in policy selection is the misalignment between the optimal reward and the desired behavior. In other words, the reward function may not fully encapsulate the intended characteristics of a well-performing policy. This challenge falls within the broader domain of reward engineering and shaping~\cite{dewey2014reinforcement,hu2020learning}, which has been widely studied but is beyond the scope of this paper. Our discussion proceeds with the given reward function as-is, and our proposed approach remains applicable across various reward function designs, regardless of whether they accurately represent the desired behavior.
\end{remark}

\begin{figure}[t]
    \centering
    \includegraphics[trim={1.cm 1.8cm 3.5cm 2cm}, clip, width=0.99\linewidth]{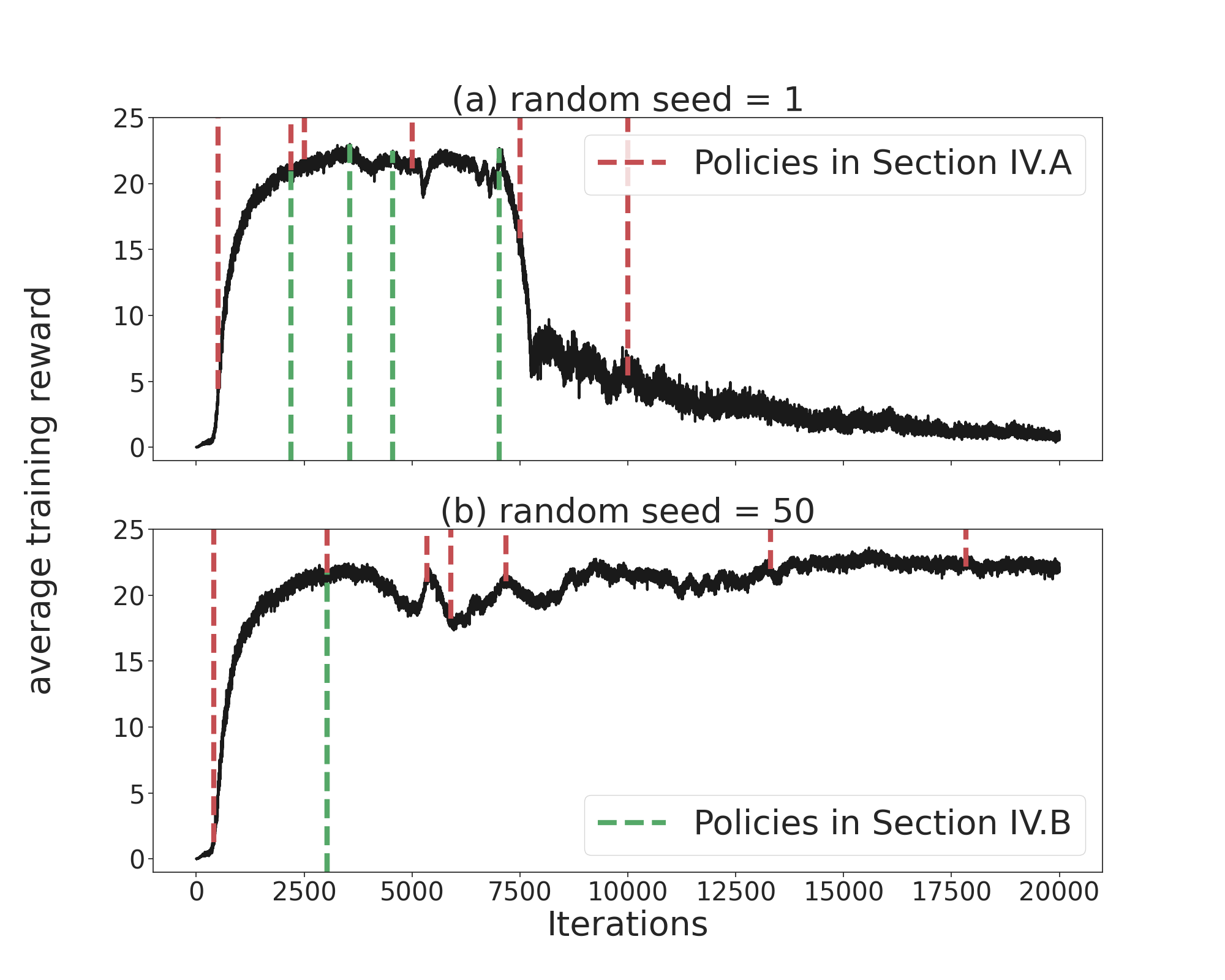}
    \caption{The average training reward of training a locomotion policy using RL for the Unitree G1 robot in Issac Gym simulator~\cite{isaacgym}. The program is a re-execution of~\cite{unitreegithub} with two different random seeds (1 and 50). The highlighted policies (with red and green dashed segments) are further analyzed in Section~\ref{sec:exp} later.}
    \label{fig:train_logs}
\end{figure}

\subsection{Literature Review}\label{sec:intro:liter}
The primary research focus on addressing the sim-to-real gap and policy transfer has largely been centered on the \emph{pre-convergence stage}, i.e., the phase in which the policy is still actively improving during training. At a high level, existing approaches can be categorized into three main philosophies aimed at improving sim-to-real transfer.  

The first category involves increasing variability in the training environment by introducing randomness in visual and geometric parameters~\cite{tobin2017domain,sadeghi2017cad2rl,chebotar2019closing}, physical dynamics~\cite{peng2018sim,chebotar2019closing,hwangbo2019learning,song2023reaching}, different fidelity levels, and causality analysis~\cite{huang2023went}. By exposing the model to a broad range of simulated conditions, these methods encourage the learning of more generalized policies, increasing the likelihood of successful real-world deployment~\cite{tobin2017domain,andrychowicz2020learning}. This category is closely associated with \emph{domain randomization} and related techniques.  

The second category focuses on minimizing the sim-to-real gap as much as possible. This includes efforts to develop high-fidelity simulators that better approximate real-world physics~\cite{todorov2012mujoco, isaacgym, dosovitskiy2017carla}. At the algorithmic level, researchers have explored techniques to understand and reduce the discrepancy, such as addressing Simulation Optimization Bias~\cite{muratore2018domain,muratore2021assessing}, computing the Sim-vs-Real Correlation Coefficient~\cite{kadian2020sim2real}, and other predictive models~\cite{zhang2020predicting,muratore2022robot}. Many of these approaches rely on additional assumptions, such as policy convergence to an optimal solution~\cite{muratore2018domain,muratore2021assessing} or system linearity~\cite{kadian2020sim2real}.  

The third category, rather than seeking broad coverage or precise simulation accuracy, focuses on \emph{adversarial testing}. This involves deliberately designing biased testing conditions that emphasize system vulnerabilities~\cite{shi2024rethinking,capito2020modeled,weng2023comparability,pinto2017robust,mandlekar2017adversarially,jiang2021simgan,lee2018spigan,pedersen2020grasping,zhang2022scgan}. The underlying expectation is that a policy capable of withstanding \emph{adversarial} disturbances will also perform well in less challenging real-world settings~\cite{koren2018adaptive,moss2020adaptive}. However, this assumption does not always hold. Recent studies have highlighted the limitations of adversarial testing, both empirically~\cite{lechner2021adversarial,lechner2023revisiting} and theoretically~\cite{weng2023comparability}, underscoring its challenges as a sole strategy for ensuring robust sim-to-real transfer. 

In practice, as we will see in Section~\ref{sec:exp}, expanding coverage, precisely matching, or introducing adversarial variations of the real-world environment all fail to guarantee strong real-world performance. While these efforts can, to some extent, enhance a policy’s overall robustness, they do not directly translate to improved performance under specific real-world conditions. More importantly, all of the above approaches are applied \emph{before} the policy reaches the commonly accepted \emph{convergence} stage, where training curves typically stabilize while still exhibiting oscillations (as shown in Fig.~\ref{fig:train_logs}). As a result, these methods fail to eliminate the heuristic trial-and-error process often required when selecting policies for real-world deployment, as outlined at the beginning of this paper. 

\subsection{Main Contributions}
Inspired by the aforementioned gaps, particularly the lack of consideration for post-convergence sim-to-real transfer and the absence of formal analysis for policy selection at this stage, this work introduces a structured approach to address these challenges.

Given the policy's performance in a specific simulator, we propose modeling an importance-weighted distribution and provably characterizing its worst-case performance under a bounded distributional divergence constraint. We show that this worst-case transference can be formulated as a convex Quadratic Constrained Linear Programming (QCLP) problem, ensuring a globally optimal solution. Furthermore, we show that this worst-case estimate serves as a more reliable predictor of real-world performance compared to direct estimates from the simulator or other existing approaches. The above theoretical analysis is further demonstrated empirically across a range of sim-to-real scenarios, including locomotion tasks under disturbed and undisturbed conditions. 

\section{PRELIMINARIES AND PROBLEM FORMULATION}
\textbf{Notation: } The set of real and positive real numbers are denoted by $\R$ and $\R_{>0}$ respectively. If $X$ is a set, $|X|$ is its cardinality. If $x$ is a scalar in $\R$, $|x|=\norm{x}_1$ denotes its absolute value. $D_{KL}(P || Q)$ denotes the Kullback-Leibler (KL) divergence of a certain distribution $P$ from a certain distribution $Q$. 

Consider a robot operating in the state space $S \subset \R^n$ (and assume $S$ is finite but could be of extremely large cardinality) and action space $U \subset \R^m$ with a certain control policy $\pi : S \rightarrow U$. It further follows a certain dynamics as
\begin{equation}\label{eq:dyn}
    f(t+1) = f(s(t), \pi(s(t))).
\end{equation}
For a certain distribution of initial states $s_0 \sim p_0(S)$ and for sufficiently many steps of propagation, \eqref{eq:dyn} renders a stable distribution $p$. During a typical sim-to-real transfer of a trained policy $\pi$, the system dynamics $f$ often differ due to discrepancies between the simulation and real-world environments, as well as additional disturbances and uncertainties not explicitly captured in \eqref{eq:dyn}. As a result, these differences lead to the emergence of distinct stationary distributions in simulation and reality.

\subsection{The performance measure with Monte-Carlo sampling}
While RL traditionally evaluates policy performance within the Markov Decision Process (MDP) framework using a predefined discount factor, the performance measure considered in this study follows a more general setting. Here, the policy is evaluated as it stands, without discounting, over a sufficient number of samples to provide a direct and unbiased assessment of its effectiveness. This approach also better aligns with standard testing and performance evaluation and management practices~\cite{al2019risk,capito2024repeatable}. 

Given the finite support $S$ (with a little abuse of notation, the sample space mentioned here could be the same with the state space $S$ mentioned above or a certain subset of it) and any stationary distribution $p$. This paper considers a class of performance measure
\begin{equation}\label{eq:m_p_star}
    \E_p[\psi] = \sum_{\s \in S} \psi(\s)p(\s) 
\end{equation}
with a certain performance metric $\psi: S \rightarrow M \subset \R$ that is monotonically increasing as the performance becomes better (or worse). For example, if $\psi: S \rightarrow \{0,1\}$ with 1 indicating the failure, \eqref{eq:m_p_star} becomes the probability for failure states to occur, i.e., the risk. In the RL settings, $\psi$ gives a single-step reward for any given $s$ and \eqref{eq:m_p_star} is the expected reward. 

Furthermore, let $\Phi_p$ be a set of samples from $S$ following the distribution $p$. The estimate of the measure in \eqref{eq:m_p_star} is
\begin{equation}\label{eq:m_p}
    \hat{\E}_p[\psi] = \frac{1}{|\Phi_p|}\sum_{\s \in \Phi_p:\s \sim p(S)} \psi(\s) p(\s),
\end{equation}
and it is immediate that $\lim_{|\Phi_p| \rightarrow +\infty} \hat{\E}_p[\psi] = \E_p[\psi]$~\cite{chatterjee2018sample}.

In practical empirical estimation, maintaining a significantly large $|\Phi_p|$ is often infeasible. The statistical significance and accuracy of \eqref{eq:m_p} are typically justified using the relative half-width (RHW) of the estimations, evaluated against a predetermined threshold $s_r \in \R _{>0}$ provided within a specified confidence interval~\cite{zhao2017accelerated,feng2020testing,feng2021intelligent}.

\subsection{Importance sampling}
In practice, especially when the performance measure is sparse (e.g., risk estimation), certain states can be exceedingly rare to encounter through standard Monte Carlo sampling, as can be inferred from \eqref{eq:m_p}. For example, in safety-critical applications, unsafe or failure states, which are crucial for evaluating safety performance measures, may rarely occur in typical simulations. This sparsity leads to poor convergence and unreliable estimates when using Monte Carlo methods.

To address the above issue, importance sampling has been proposed as an effective strategy~\cite{gutjahr1997importance,zhao2016accelerated,feng2021intelligent}. By biasing the exploration toward these rare but critical events, importance sampling enhances the probability of sampling from regions of the state space that are otherwise underrepresented. Despite this biased exploration, importance sampling ensures that the final estimate of the performance measure remains unbiased by appropriately adjusting the weights of the sampled trajectories. This reweighting is achieved using the likelihood ratio between the biased and original distributions, thereby providing a statistically accurate estimate of the performance measure under the true distribution while efficiently handling rare events.

Consider the importance distribution $q$ over the same support $S$ as mentioned above. Let $\Phi_q$ be the set of samples from $S$ under $q$, we have
\begin{equation}\label{eq:m_q}
    \hat{\E}_q[\psi] = \frac{1}{|\Phi_q|} \sum_{\s \in \Phi_q:\s \sim q(S)} \psi(\s) \frac{p(\s)}{q(\s)}
\end{equation}
The above estimate is unbiased as $\E_p[\psi] = \lim_{|\Phi_p|\rightarrow \infty} \hat{\E}_p[\psi] = \lim_{|\Phi_q|\rightarrow \infty} \hat{\E}_q[\psi]$.

In the practice of sim-to-real transfer, the desired stationary distribution in the real-world can be represented as $p$, while the simulator-based evaluation, supplied with the same control policy executed in the simulator, yields an importance-weighted distribution $q$. If both $p$ and $q$ were known, one could directly compute the necessary adjustments to improve transferability. However, in practice, neither the real-world distribution $p$ nor the exact stationary distribution $q$ in the simulator is explicitly known—the latter can be approximated, but not precisely defined. 

Existing post-convergence sim-to-real efforts primarily focus on generating different variants of $q$ using high-fidelity simulators, adversarial testing, or domain randomization (adapted from the pre-convergence training stage). These methods aim to improve robustness by diversifying simulated experiences, yet they lack a principled way to quantify performance guarantees in real-world deployment—and, as we shall see later, they do not necessarily serve as reliable indicators of the policy's actual performance in $p$.

This motivates us to introduce an additional weighted distribution $\rho$, positioned within a quantifiable neighborhood of $q$, which induces a worst-case performance estimate. We theoretically prove and empirically demonstrate that this estimate provides a reliable predictor of the policy’s performance in $p$, performing at least as well as $q$ and often yielding better predictive accuracy. Further details will be discussed in Section~\ref{sec:main}.

\subsection{Problem formulation}
The sim-to-real policy transfer at the post-convergence stage aims to make an informed and rational selection of the policy (or policies) to deploy in real-world experiments, which can be costly, unsafe, and unpredictable. Fundamentally, this process can be framed as a \emph{ranking prediction problem}, where the goal is to identify the most promising policies based on a certain indication of their expected real-world performance.

Consider an arbitrary pair of distinct policies, $ \pi_1 $ and $ \pi_2 $, which induce stationary distributions in a certain real-world environment, denoted as $ p_1 $ and $ p_2 $, over certain subsets of $ S $, $ \Phi_1 $ and $ \Phi_2 $. Notably, the distributions $p$ and the corresponding subsets $ \Phi $ are primarily unknown. Suppose the relative correlation between the expected performance under these distributions satisfies $\E_{p_1}[\psi] < \E_{p_2}[\psi]$, the post-convergence sim-to-real policy transfer problem seeks to identify performance indicators $ I_1 $ and $ I_2 $ such that the relative ranking order is preserved, i.e., $I_1 < I_2$.

For a certain policy trained through RL with or without some of the mentioned post-convergence sim-to-real techniques in Section~\ref{sec:intro:liter}, the determination of its real-world performance is often related to its training environment. Assume such policy forms a stationary distribution $p$ over $\Phi \subset S$.
The domain randomization techniques discussed in Section~\ref{sec:intro:liter} often result in a significantly expanded state space, with $\Phi$ being a subset. However, the resulting state distribution may or may not accurately approximate $p$ within the shared region. A similar pattern is observed in adversarial testing methods, where the state distribution shifts toward riskier states. Whether these states yield lower rewards depends heavily on the construction of the reward function. Another approach is to approximate $p$ directly to the best extent possible. However, preserving the relative ranking order of policies requires restrictive conditions on the similarity between $ p $ and the approximated distribution $q$ in simulation.

To this end, it should become conceptually clear—and will be empirically demonstrated in Section~\ref{sec:exp}—that none of the aforementioned distribution $q$ and their corresponding reward expectations serve as effective predictors of real-world reward expectations under $p$. The proposed method aims to address this issue by introducing a more robust predictor for real-world performance.

\section{MAIN METHOD}\label{sec:main}

\begin{figure}
    \centering
    \includegraphics[trim={1.5cm 5.5cm 15.5cm 2cm}, clip, width=0.9\linewidth]{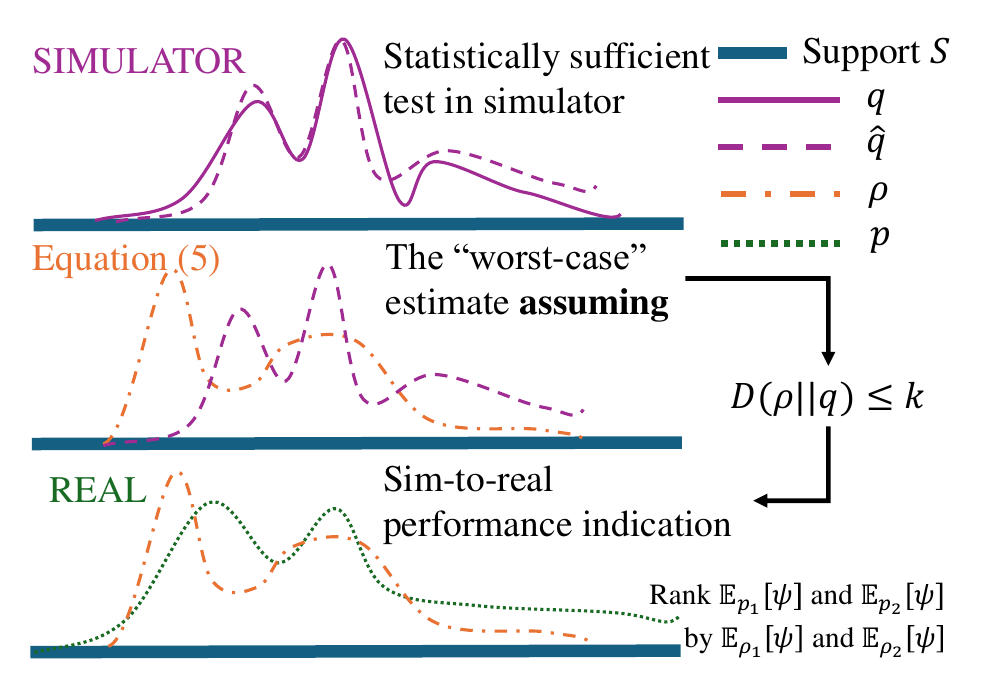}
    \caption{An overview of the main steps in the proposed post-convergence sim-to-real performance indication.}
    \label{fig:overview}
\end{figure}

The proposed method complements existing sim-to-real efforts at the post-convergence stage. It is specifically applied (i) \emph{after} conducting sufficient tests in the simulator under given conditions, yielding a rendered stationary distribution $q$, and (ii) \emph{before} deploying one or more policies on real robots in real-world environments. A conception overview of the proposal is shown in Fig.~\ref{fig:overview}.  

\subsection{The KL-constrained worst-case performance}\label{sec:mod2}
If a sufficient number of tests are conducted in the simulator, the underlying stationary distribution of states $q$ can be estimated with $\hat{q}$ using the collected samples, along with the performance estimation $\hat{\E}_q[\psi]$ provided with a confidence interval $c \in (0,1)$ and a relative half-width value $s_r$ ensuring that, with probability $1-c$, $|\hat{\E}_q[\psi]-\E_q[\psi]| \leq s_r \cdot \hat{\E}_q[\psi]$.
By setting $s_r$ and $c$ arbitrarily close to zero, $\hat{q}$ converges to $q$. Consequently, we do not distinguish between them for the remainder of this paper. This technique is similarly applied in Section~\ref{sec:exp} to achieve a sufficiently accurate approximation of various humanoid robot locomotion policies in real-world experimental settings.

We further propose to consider the worst possible performance measure under the distribution $\rho$ assuming the difference between $\rho$ and $q$ is bounded as $D_{KL}(\rho || q) \leq k, \forall \rho$. Given the same set of states $\Phi_q$ sampled following the underlying distribution of $q$, one can estimate the ``worst-case" performance measure across a class of possible $\rho$ distributions through \eqref{eq:optml_r}. Note there is no strict requirement for $k$ to be small, yet a large $k$ would typically lead to a worse worst-case estimate.

\begin{subequations}\label{eq:optml_r}
    \begin{align}
    \underset{\rho}{\max} \quad & \sum_{\s \in \Phi_q}\psi(s)\rho(s) \\
    \text{s.t. }\quad  & \sum_{\s \in S}\rho(s) = 1 \label{eq:optml_r_sum} \\
    \quad & \rho(\s) \in [0,1], \forall \s \in S \label{eq:optml_r_pb} \\
    \quad & \sum_{\s \in S} \frac{\rho^2(\s)}{q(\s)} - \rho(\s) \leq k \label{eq:optmal_r_kl} 
    \end{align}
\end{subequations}
Note the constraints \eqref{eq:optml_r_sum} and \eqref{eq:optml_r_pb} are both linear (in $\rho$). The maximizer is selected under the assumption that a higher $\psi$ value is preferable (e.g., for risk maximization). Conversely, a minimizer can be used in cases where lower values are desirable (e.g., for reward minimization). In both settings ($\min$ or $\max$), the (positive definite) quadratic constraint \eqref{eq:optmal_r_kl} is coming from the upper bound of KL divergence as $D_{KL}(\rho||q) \leq \sum_{\s \in S} \frac{{\rho}^2(\s)}{q(\s)} - \rho(\s)$. It is adopted to make the optimization problem \eqref{eq:optml_r} a convex QCLP problem with trackable solutions and guarantees to the global optimal~\cite{boyd2004convex}, while also ensuring the bounded distributional divergence of $D_{KL}(\rho\mid\mid q)\leq k$. 

Moreover, the practical efficiency and feasibility of \eqref{eq:optml_r} also depend on the dimension of $q$, which, in the worst case, corresponds to the cardinality of $S$. Most solvers, depending on hardware specifications, can handle $|S|$ up to within the range of $10^6$ to $10^9$ , which is sufficient for all cases studied in Section~\ref{sec:exp}. Consequently, a uniform discretization approach (i.e., rounding all values to a fixed number of decimal places) is adequate. However, in more complex settings, more advanced techniques such as adaptive discretization may be necessary. A detailed discussion of these methods is beyond the scope of this paper and is left for future work.

\subsection{Worst-case expectation improves ranking robustness}
We further prove that the KL-constrained worst-case expectation, derived through \eqref{eq:optml_r}, serves as a more reliable indicator of sim-to-real performance compared to direct outcome comparisons from the simulator or adversarial variations within the simulation environment.

\begin{theorem}\label{thm:rho_rocks}
    Let $p_1$ and $p_2$ be two unknown probability distributions satisfying $\E_{p_1}[\psi] > \E_{p_2}[\psi]$ for some performance measure function $\psi$. Let $q_1$ and $q_2$ be approximations of $p_1$ and $p_2$, respectively, with unknown discrepancies.
    Let $\rho_1$ and $\rho_2$ be the worst-case distributions by \eqref{eq:optmal_r_kl} subject to $D_{KL}(\rho_i\mid\mid q_i)\leq k, i=1,2$. We have $\mathbb{P}(\E_{\rho_1}[\psi] > \E_{\rho_2}[\psi]) \geq \mathbb{P}(\E_{q_1}[\psi] > \E_{q_2}[\psi])$, i.e., the worst-case expectation is at least as good, and typically strictly better, at ranking $\E_{p_1}[\psi]$ and $\E_{p_2}[\psi]$ than using $\E_{q_1}[\psi]$ and $\E_{q_2}[\psi]$ directly.
\end{theorem}
\begin{proof} The first step of proving the above theorem is to show that the worst-case estimate from \eqref{eq:optml_r} reduces the variance, i.e., $\text{Var}_{\rho_i}(\psi) \leq \text{Var}_{q_i}(\psi)$. Following the Donsker-Varadhan's variational representation of KL-divergence, we have $\rho_i(s) = \frac{q_i(s)\cdot e^{-\lambda\psi(s)}}{\E_{q_i}[e^{\lambda\psi}]}$ for a Lagrange multiplier $\lambda$ that enforces the KL-divergence constraint. This reweighting downweights high values of $\psi(s)$, effectively shifting probability mass toward values of $\psi$ that are closer to the mean. As a result $\E_{\rho_i}[\psi^2]-(\E_{\rho_i}[\psi])^2 \leq \E_{q_i}[\psi^2]-(\E_{q_i}[\psi])^2$ with equality only if $q_i$ is already the worst-case shifted distribution, meaning no reweighting occurs. This complete the first step showing that using $\rho$ reduces the fluctuation in expectation estimates. 
    
The second step seeks to prove the worst-case estimate $\rho$ is closer to $p$ than $q$ in terms of expectation. To proceed, for any $i$, let $|\E_{p_i}[\psi] - \E_{q_i}[\psi]| = \delta_q$ and $|\E_{p_i}[\psi] - \E_{{\rho}_i}[\psi]| = \delta_{\rho}$. Using a general variance-based deviation bound, we have
    \begin{equation}
        |\E_{p_i}[\psi] - \E_{X}[\psi]| \leq C\sqrt{\text{Var}(\E_X[\psi])}
    \end{equation}
    for a distribution X (either $q_i$ or $\rho_i$), and for some positive constant $C$ that is problem dependent. Since from step one we have shown $\text{Var}_{\rho_i}(\psi) \leq \text{Var}_{q_i}(\psi)$, it immediately follows that $\sqrt{\text{Var}_{\rho_i}(\psi)} \leq \sqrt{\text{Var}_{q_i}(\psi)}$, we then have
    \begin{equation}
        |\E_{p_i}[\psi] - \E_{{\rho}_i}[\psi]| \leq |\E_{p_i}[\psi] - \E_{q_i}[\psi]| \Rightarrow \delta_{\rho} \leq \delta_q.
    \end{equation}

    Finally, let $Z = \E_{q_1}[\psi] - \E_{q_2}[\psi]$ and $Z' = \E_{\rho_1}[\psi] - \E_{\rho_2}[\psi]$. The probability of correctly ranking $p_1$, $p_2$ becomes $\mathbb{P}(Z>0)$ and $\mathbb{P}(Z'>0)$, respectively. Given $\delta_{\rho} \leq \delta_q$ from step two, it is immediate that $\text{Var}(Z') \leq \text{Var}(Z)$. By Chebyshev’s inequality, for any $t>0$,
    \begin{equation}
        \mathbb{P}(|Z'-\E[Z']|\geq t) \leq \frac{\text{Var}(Z')}{t^2} \leq \frac{\text{Var}(Z)}{t^2}.
    \end{equation}
    That is, $Z'$ is more concentrated around its expectation than $Z$. As a result, $\mathbb{P}(Z'>0) \geq \mathbb{P}(Z>0)$. This completes the proof.
\end{proof}

Note that Theorem~\ref{thm:rho_rocks} guarantees that predictability will not degrade uniformly, without requiring the gap between $p_i$ and $q_i$ be small (i.e., without assuming a small sim-to-real gap). In practice, while a small sim-to-real gap is always desirable, as it further reduces absolute noise, the proposed method remains effective even when this gap is significant.

Moreover, determining the optimal value of $k$ that maximizes sim-to-real performance prediction is challenging. While uniform improvement is guaranteed for all $k$, the extent to which the worst-case estimate enhances ranking accuracy remains uncertain and varies with $k$, often in a nonlinear manner (e.g., Fig.~\ref{fig:nominal-scc}). Also, numerical errors in practical computations—such as those arising from numerical solvers used in \eqref{eq:optml_r}—can further introduce instability or suboptimal improvements. For specific applications, users can select $k$ and other related parameters based on the acceptable deviation for their use case or leverage prior knowledge of real-world and simulator fidelity constraints to make an informed decision. A more in-depth discussion of application-specific considerations falls outside the scope of this paper and remains an avenue for future research.

\section{EXPERIMENTS}\label{sec:exp}
This section demonstrates the effectiveness of the proposed worst-case performance estimator as an empirical predictor of real-world performance for RL policies trained in simulators. To begin, we trained a locomotion policy for the Unitree G1 robot using the exact program and hyperparameters specified by~\cite{unitreegithub}, with two different random seeds: 1 and 50. The training results are presented in Fig.~\ref{fig:train_logs}.

\begin{figure}
    \vspace{2mm}
    \centering
    \includegraphics[trim={3.55cm 6.5cm 16.cm 1.8cm}, clip, width=0.82\linewidth]{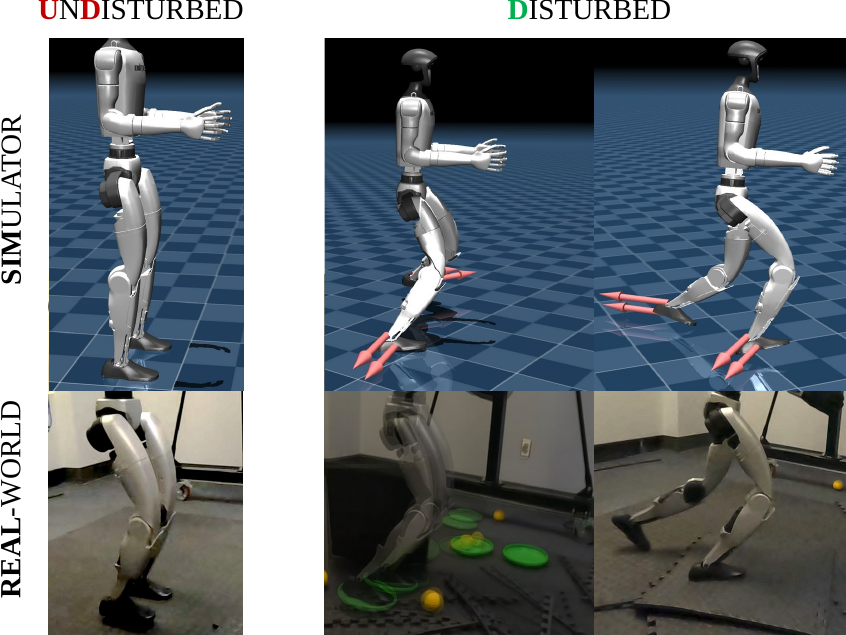}
    \caption{Visual representations of the undisturbed and disturbed testing setups in simulator and real-world. 13 policies (see red segments in Fig.~\ref{fig:train_logs}) and 5 policies (see green segments in Fig.~\ref{fig:train_logs}) are selected, respectively, for the undisturbed and disturbed tests.}
    \label{fig:env_visual}
\end{figure}

We further select two groups of policies from the two rounds of training, highlighted by red and green dashed segments in Fig.\ref{fig:train_logs}, respectively. These two groups are then deployed in two distinct environments, referred to as undisturbed (red) and disturbed (green). Visual representations of these environments are shown in Fig.\ref{fig:env_visual}. The simulator is MuJoCo~\cite{todorov2012mujoco} for all tests (except for the training reward from Fig.~\ref{fig:train_logs}, which is in Isaac Gym~\cite{unitreegithub,isaacgym}). In the real-world undisturbed test, the robot executes an in-place stepping motion, while the corresponding simulated environment attempts to replicate the same conditions as closely as possible. All hyperparameters, policies, and controllable variables remain consistent between simulation and reality. For the disturbed test, the robot navigates an uneven, soft, and slippery terrain composed of foam pads of different thickness, baseballs, and frisbee plates of two different sizes. Notably, the corresponding simulated environment does not include these exact objects or terrain features. Instead, disturbances are introduced by directly applying external forces (up to 20 N) to the ankle joints and foot bodies (see Fig.~\ref{fig:env_visual}). For each group and its corresponding testing environment, policies are first evaluated in the real-world to obtain their expected reward using a predefined reward function. This evaluation is conducted with a confidence interval of $0.05$ and across various RHW threshold values. The policies are then ranked in ascending order based on their obtained reward expectations. This ranked sequence serves as the ground-truth reference for assessing the predictive accuracy of different sim-to-real performance indicators, including the proposed method in this paper.

\begin{figure}[t]
    \centering
    \includegraphics[trim={1.cm 1.5cm 3.5cm 2cm}, clip, width=0.9\linewidth]{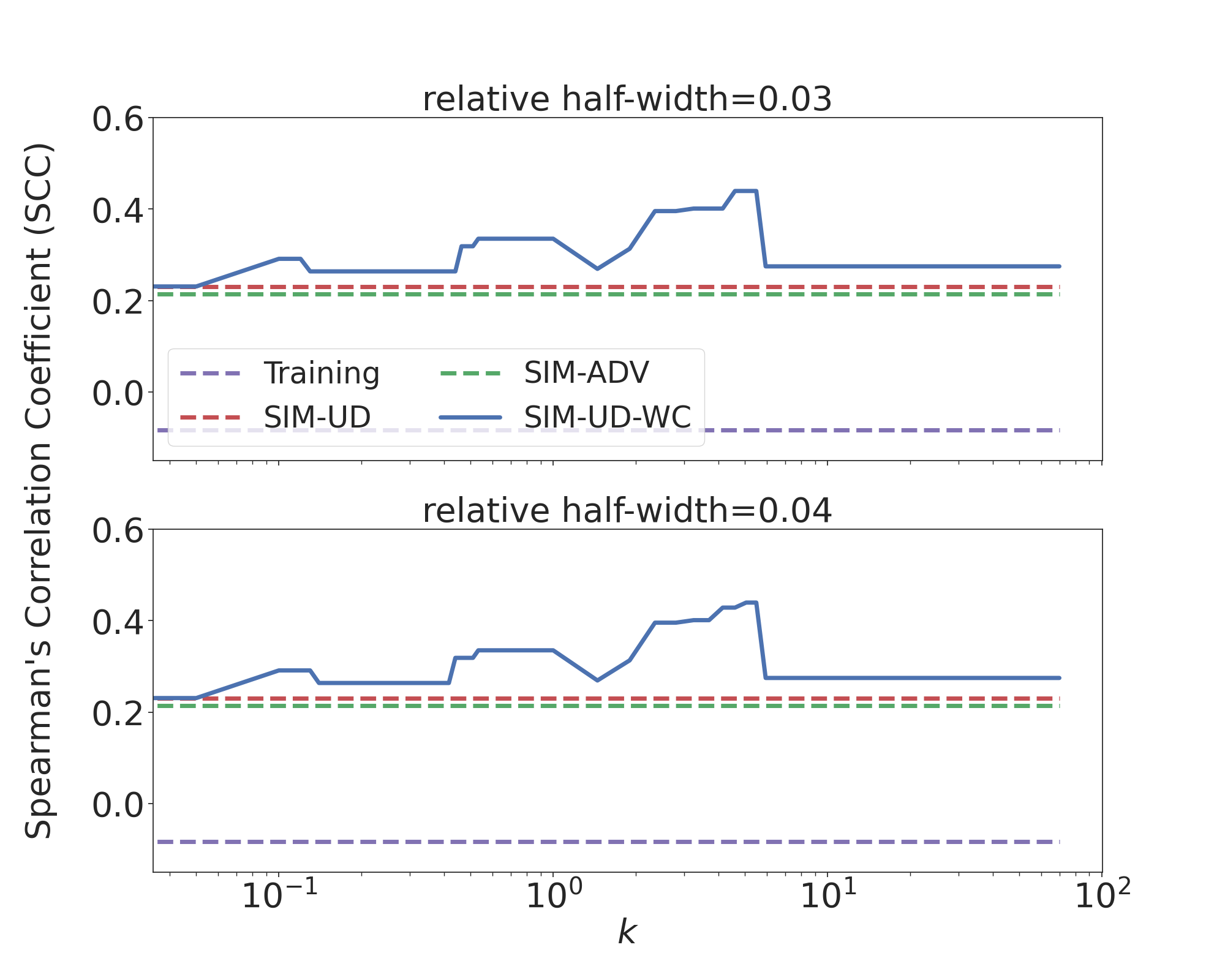}
    \caption{Comparison of predicted policy rankings from various sim-to-real performance indicators against real-world evaluated rankings in the undisturbed testing group, measured using Spearman’s correlation coefficient (SCC) across two different RHW levels (confidence level $0.05$). The proposed worst-case estimates exhibit some variation depending on the specified KL-divergence bound $k$. In contrast, other indicators, such as direct simulation-based estimates and adversarial simulation evaluations, remain constant and appear as flat lines.}
    \label{fig:nominal-scc}
\end{figure}

\begin{figure}
    \centering
    \includegraphics[trim={1.cm 1.5cm 3.5cm 2cm}, clip, width=0.9\linewidth]{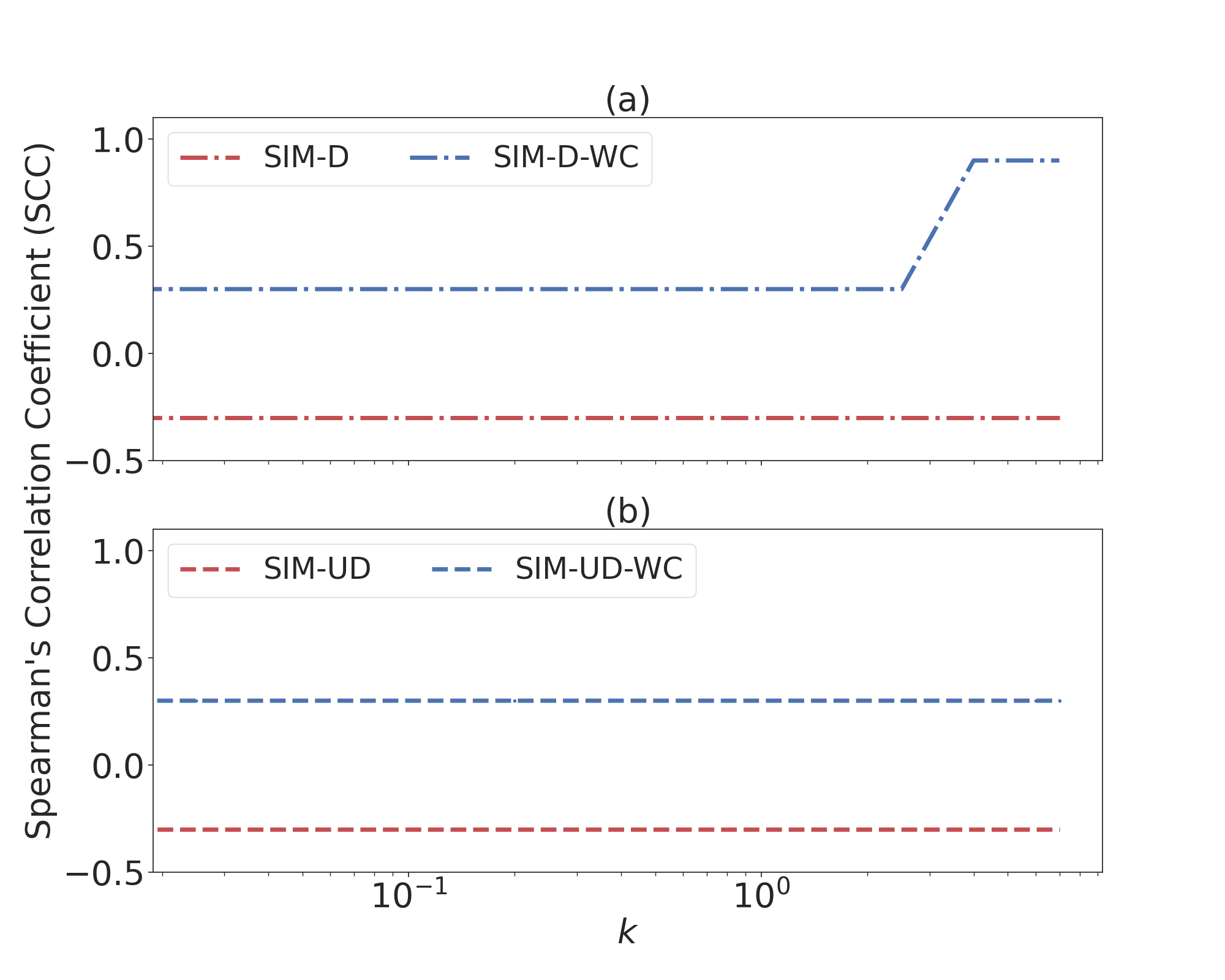}
    \caption{Comparison of predicted policy rankings from various sim-to-real performance indicators against real-world evaluated rankings in the disturbed testing group using the stability reward function $r_2$, measured using Spearman’s correlation coefficient (SCC) at the RHW level of $0.03$ (with confidence level $0.05$): (a) and (b) are using the \textbf{disturbed} and \textbf{undisturbed} testing environment, respectively, in simulation and its corresponding worst-case estimates to indicate \textbf{real-world performance under disturbances}.}
    \label{fig:terrain-scc}
\end{figure}

The comparison between ranked sequences is performed using Spearman’s rank correlation coefficient~\cite{sedgwick2014spearman}, which measures the monotonic relationship between two rankings. This coefficient takes values in $[-1,1]$, where $1$, $0$, and $-1$, indicates perfect rank match, no correlation, and completely inverted ranking, respectively.

\begin{figure*}[ht]
    \centering
    \includegraphics[trim={1.cm .3cm 5.cm 1.cm}, clip,width=0.88\linewidth]{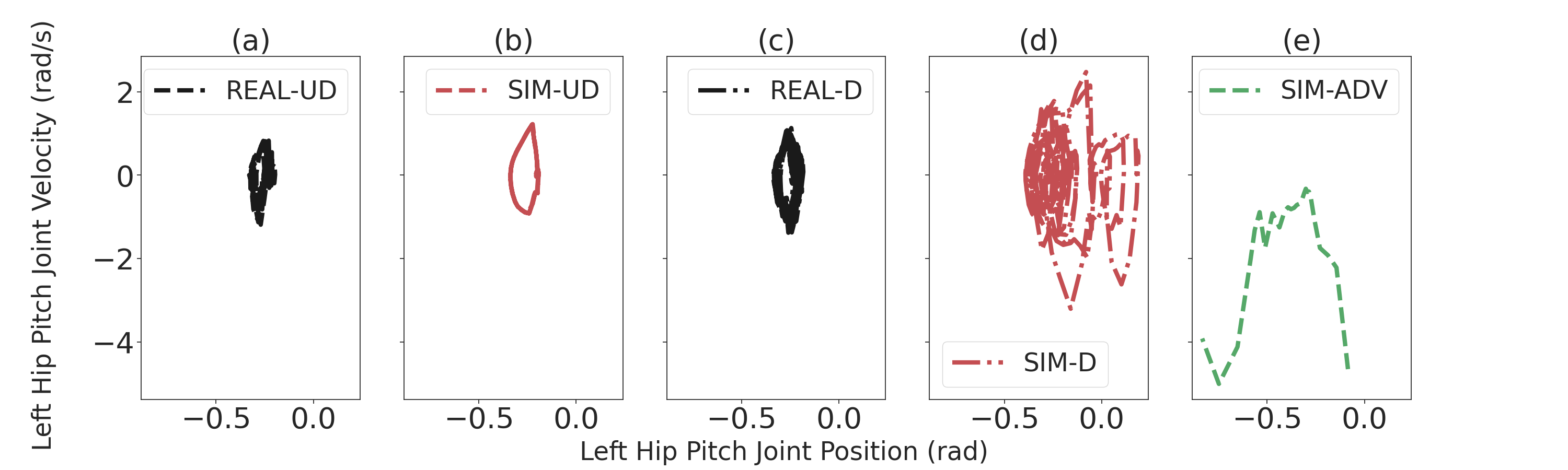}
    \caption{The limit walking cycle of the same policy at the left hip pitch joint for $10$ seconds (with (e) being an exception as the policy cannot ``survive that long" against the adversarial setup) under various simulated tests and real-world experiments.}
    \label{fig:phase-portrait}
\end{figure*}

The reward functions used in this study are largely derived from the original reward function used during policy training~\cite{unitreegithub}. However, terms related to the pelvis center of mass (CoM) positions and linear velocities have been removed, as these states lack direct sensor measurements on the real-world robot. While estimations through filtering and kinematic constraints are possible, the measurement noise is too significant to ensure rigorous and fair testing. The key terms used in the reward function design for each state, denoted as $r=[r_{x}, r_{\dot{x}}, r_j, r_{\dot{j}}, r_h, r_a]^T$, includes the following components: $r_{x}$ (encouraging an upright torso orientation), $r_{\dot{x}}$ (penalizing excessive torso angular velocity), $r_{j}$ (penalizing extreme joint positions), $r_{\dot{j}}$ (penalizing large joint velocities), $r_{h}$ (penalizing large hip angles), and an alive bonus $r_{a}$, along with the weights. The final reward is computed as a weighted linear combination of these terms, given by $\omega^T r$. Detailed descriptions of each term and their respective scaling factors can be found in~\cite{unitreegithub}. This particular formulation is referred to as the performance reward, $r_1$, throughout the remainder of this paper. Additionally, we introduce a reward variant that emphasizes the stability performance of the robot. This variant retains only $r_x$ and $r_{\dot{x}}$ from the performance reward formulation and is referred to as the stability reward, $r_2$. The performance reward $r_1$ is primarily used in the undisturbed tests described in Section~\ref{sec:exp:ud}, while both $r_1$ and $r_2$ are considered in the disturbed tests in Section~\ref{sec:exp:d}.

\begin{remark}
    The reward function also plays a crucial role in reducing the dimensionality of the state space, facilitating a more efficient solution to \eqref{eq:optml_r}. Generally, the state space $S$ would need to encompass all direct sensor readings contributing to the reward, including joint positions, velocities, torso CoM states, and more. This could lead to a distribution $q$ of significantly high dimensionality, making optimization computationally demanding. However, in this specific formulation of expected reward estimation, the reward function serves as a predefined mapping from the complex original state space to a reduced reward state space. Consequently, performing distributional analysis over this reduced reward state space is equivalent to analyzing the original high-dimensional space but with significantly lower computational complexity. This reduction enables a more tractable and efficient optimization process without compromising accuracy.
\end{remark}

\subsection{Undisturbed locomotion}\label{sec:exp:ud}
In addition to the real-world (REAL-UD) and simulated (SIM-UD) setups mentioned earlier, the policies in the undisturbed locomotion group are also evaluated using an adversarial testing framework similar to the one proposed in~\cite{shi2024rethinking}. Each policy is alongside a dedicated ``attacker" trained to introduce disturbance forces at all joints and body links within a limited magnitude. The trained attacker is then used to assess the policy's robustness, a process referred to as SIM-ADV for the remainder of this section.

Fig.~\ref{fig:nominal-scc} summarizes a variety of comparisons made between the real-world experiment based policy rankings and various other indicators, including the proposed worst-case estimate. Reward function used in these studies is $r_1$. Specifically, the training reward is directly obtained from the policy’s training process (similar to the values shown in Fig.~\ref{fig:train_logs}), and its SCC against the real-world ranking is negative, indicating that the two ranked sequences exhibit no meaningful correlation. SIM-UD shows a slightly better agreement with real-world rankings than SIM-ADV, though both remain below 0.25 across all RHW thresholds. In contrast, the proposed worst-case estimate, using SIM-UD results as $q$, validates the theoretical guarantees from Theorem~\ref{thm:rho_rocks} and consistently outperforms, or at worst matches, the ranking predictions made by the aforementioned indicators. In the best-case scenario ($k=2.5$), the SCC value reaches as high as 0.55, demonstrating a significantly stronger predictive capability.

\subsection{Disturbed locomotion}~\label{sec:exp:d}
Since the disturbed tests pose a higher risk in real-world experiments, a different group of policies is selected for evaluation (highlighted in green in Fig.~\ref{fig:train_logs}). Their expected performance rankings are determined using both $r_1$ and $r_2$ across different RHW threshold levels, with a confidence level of $0.05$. Among the five policies in this study, when evaluated using the performance reward $r_1$, the rankings obtained from SIM-D and REAL-D (i.e., real-world performance) exhibit an SCC value of 1.0 across all cases, indicating perfect alignment. As expected, the worst-case estimate does not degrade the ranking accuracy, consistently achieving an SCC of 1.0 across various selections of $k$. This outcome is attributed to the construction of $r_1$, which balances stability, efficiency, and pose alignment.

To further investigate (under a ``disagreed" sim-to-real setting), we evaluate an alternative reward function, $r_2$, which specifically emphasizes stability performance. The results are presented in Fig.~\ref{fig:terrain-scc}. When using direct simulated performance in both disturbed (Fig.\ref{fig:terrain-scc}(a)) and undisturbed (Fig.\ref{fig:terrain-scc}(b)) settings, both indicators yield an SCC value of $-0.3$ when compared to real-world performance under disturbances. This indicates that these indicators are not correlated with real-world performance rankings (REAL-D) at all. In contrast, SIM-UD-WC consistently achieves an SCC value of $0.3$, demonstrating a significant $0.6$ improvement over SIM-UD. Furthermore, SIM-D-WC consistently produces positive SCC values, including some almost-perfect alignment with REAL-D performance rankings ($\text{SCC}=0.9$).

\begin{figure}
    \centering
    \includegraphics[trim={.5cm 1.5cm 3.cm 3.cm}, clip,width=0.9\linewidth]{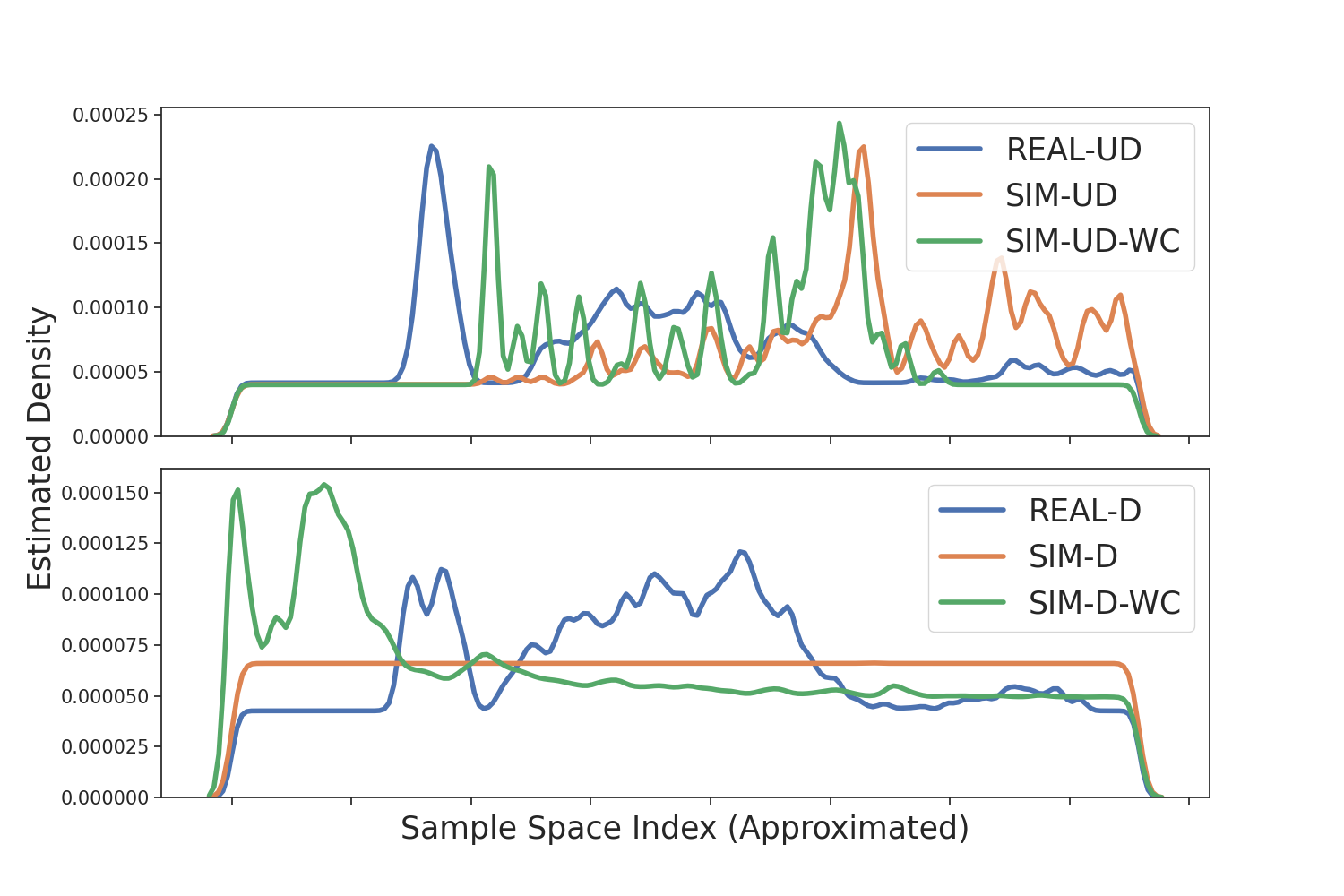}
    \caption{The kernel density estimate (KDE) plot illustrates the distributions (in the reward state space) of the same policy tested in SIM-D, SIM-UD, REAL-D, and REAL-UD, along with the corresponding worst-case estimates for $k=6.0$. Notably, this same bound improves the SCC of SIM-UD and SIM-D from $-0.3$ to $0.3$ and $0.9$, respectively, when compared to the REAL-D performance rankings. Note that the KDE provides a significantly simplified representation of the underlying distribution due to the high dimensionality of the state space, offering only a partial view of the true distributional differences across environments. As a result, while the variations in state space coverage across different environments are evident—partially illustrated in Fig.~\ref{fig:phase-portrait}—these differences may not be visually apparent in the approximated KDE analysis.}
    \label{fig:dist}
\end{figure}

Comparing subfigures (b), (c), and (d) in Fig.~\ref{fig:phase-portrait} provides insight into why SIM-UD offers a better prediction of REAL-D than SIM-D. The simulated disturbances in SIM-D force the joints to explore a much larger state space in a more random, non-periodic manner, which fail to approximate real-world variability. Similarly, SIM-ADV induces a significantly different state space coverage, further highlighting the discrepancies between different simulation strategies in capturing real-world behavior.

However, sim-to-real differences persist across corresponding subplots under both disturbed and undisturbed conditions in Fig.~\ref{fig:phase-portrait}, highlighting the inherent challenges of direct transfer of ranked policies. Fig.\ref{fig:dist} provides an alternative perspective by visualizing simplified distributions of another policy across various configurations. Each subfigure compares $p$ (real-world), $q$ (simulator), and $\rho$ (worst-case estimate at $k=6.0$) under the same disturbed and undisturbed settings. Notably, in terms of distributional alignment, the worst-case estimate $\rho$ does not necessarily move $q$ closer to $p$. However, both theoretically and empirically, it reduces the variance of the reward estimate, leading to a more stable and reliable performance predictor in sim-to-real transfer scenarios.

\section{CONCLUSIONS AND FUTURE WORK}
To the best of our knowledge, this paper makes one of the first attempts to address sim-to-real policy transfer at the \emph{post-convergence} stage. The proposed worst-case estimate, formulated through~\eqref{eq:optml_r}, paves the way for future advancements, including enhanced adversarial testing for handling more complex dynamics $f$ and state spaces $S$. Additionally, addressing the limitations imposed by the curse of dimensionality in solving~\eqref{eq:optml_r} remains an important direction for future research.

\bibliographystyle{IEEEtran}
\bibliography{bib}

\end{document}